\documentclass[a4paper,cleveref,autoref]{lipics-v2021}
\hideLIPIcs
\usepackage{amsmath,amssymb,amsthm,graphicx,cite,algorithm,algpseudocode,textcomp}
\usepackage{color}

\usepackage{subcaption}
\usepackage{graphicx}
\usepackage{booktabs}
\nolinenumbers
\newcommand{\poly}{\mathrm{poly}}

\title{Learning-Augmented Online Bipartite Matching\texorpdfstring{\\}{ }in the Random Arrival Order Model}
\titlerunning{Learning-Augmented Online Bipartite Matching}

\author{Kunanon Burathep}{Department of Computer Science, Durham University, UK}{kunanon.burathep@durham.ac.uk}{https://orcid.org/0009-0000-5262-9300}{}

\author{Thomas Erlebach}{Department of Computer Science, Durham University, UK}{thomas.erlebach@durham.ac.uk}{https://orcid.org/0000-0002-4470-5868}{}

\author{William K. Moses Jr.}{Department of Computer Science, Durham University, UK}
{william.k.moses-jr@durham.ac.uk}{https://orcid.org/0000-0002-4533-7593}{}

\authorrunning{Kunanon Burathep, Thomas Erlebach, William K. Moses Jr.}


\relatedversion{An extended abstract of this paper appears in the proceedings of the 
51st International Conference on Current Trends in Theory and Practice of Computer Science
(SOFSEM 2026).}

\begin{document}
\maketitle

\begin{abstract}
We study the online unweighted bipartite matching problem in the random 
arrival order model, with $n$ offline and $n$ online vertices, in the
learning-augmented setting: The algorithm is provided with untrusted predictions of
the types (neighborhoods) of the online vertices.
We build upon the work of Choo et al.~(ICML 2024, pp.~8762--8781) who proposed
an approach that uses a prefix of the arrival sequence as a sample to
determine whether the predictions are close to the true arrival
sequence and then either follows the predictions or uses a known baseline
algorithm that ignores the predictions and is $\beta$-competitive.
Their analysis is limited to the case that the optimal matching
has size~$n$, i.e., every online vertex can be matched.  
We generalize their approach and analysis by removing any assumptions on the size of the optimal matching while only requiring that the size of the predicted matching is at least $\alpha n$ for any constant $0 < \alpha \le 1$. 
Our learning-augmented algorithm achieves $(1-o(1))$-consistency and
$(\beta-o(1))$-robustness.
Additionally, we show that the
competitive ratio degrades smoothly between consistency and robustness
with increasing prediction error.
\end{abstract}

\section{Introduction}
Online bipartite matching is a fundamental problem in theoretical computer
science with significant applications in areas such as online advertising,
ride-sharing platforms, and resource allocation. This problem involves
a bipartite graph where one side of the graph (the offline vertices) is known in
advance, while the other side (the online vertices) arrives sequentially.
When an online vertex arrives, it must be permanently matched to an offline vertex or
left unmatched, without knowledge of future arrivals. The goal is to maximize
the size of the matching.

Traditional approaches to online bipartite matching consider the worst-case
guarantees in the adversarial arrival model, where
an adversary controls the graph's structure and the arrival sequence of online
vertices. It is easy to see that the greedy algorithm, which matches
each online vertex to an arbitrary unmatched neighbor if one exists,
is $\frac12$-competitive, and it is known that randomization
allows one to obtain an improved competitive ratio of $1-1/e\approx 0.632$~\cite{karp1990optimal}.
A setting that makes it possible to achieve an even better competitive ratio, at least $\beta\approx 0.696$,
is the random arrival order model, where the
sequence of online vertices is assumed to arrive in a uniformly random order~\cite{goel2008online,karande2011online,mahdian2011online}.

Recently, learning-augmented algorithms have emerged as a means to enhance online
algorithms by incorporating (untrusted) predictions about future inputs. These predictions
are also referred to as advice. Learning-augmented algorithms
aim to achieve near-optimal performance when the predictions are accurate while
maintaining a competitive ratio close to that of online algorithms without predictions
when the predictions are inaccurate. For learning-augmented online bipartite
matching, various types of predictions have been studied, such as
degree information~\cite{aamand2022optimal}, advice derived from reinforcement
learning~\cite{li2023learning}, and predictions of the types of online
vertices~\cite{online-bipartite-matching-imperfect-advice}. In this paper,
we build upon the work of Choo et al.~\cite{online-bipartite-matching-imperfect-advice}
who considered the latter type of predictions in the random arrival order
model. They proposed
an approach that uses a prefix of the arrival sequence as a sample to
determine whether the predictions are close to the true arrival
sequence and then either follows the predictions or uses a known baseline
algorithm that ignores the predictions and is $\beta$-competitive.
Here, $\beta$ is the best-known competitive ratio of an online algorithm
for bipartite matching without predictions in the random arrival
order model, currently $\beta\approx 0.696$~\cite{mahdian2011online}.
The analysis by Choo et al.\ is limited to the case that the optimal matching
has size~$n$, i.e., every online vertex can be matched. 

We generalize their approach and analysis by removing any assumptions on the size of the optimal matching while only requiring that the size of the predicted matching is at least $\alpha n$ for an arbitrary constant $0 < \alpha \le 1$.
Our learning-augmented algorithm achieves $(1-o(1))$-consistency and
$(\beta-o(1))$-robustness.
Additionally, we show that the
competitive ratio degrades smoothly between consistency and robustness
with increasing prediction error: We show that the competitive
ratio is at least $1 - \frac{2 L_1(p,q)}{2 \alpha + L_1(p,q)} - o(1)$,
where $L_1(p,q)$ is the $L_1$-distance between the distribution of vertex
types in the predicted arrival sequence and the true arrival sequence.

Our result addresses the following key limitations of the algorithm and analysis proposed by Choo et al.~\cite{online-bipartite-matching-imperfect-advice}. 
In cases where the predicted matching size is smaller than $\beta n$,
their approach does not achieve $(1-o(1))$-consistency even if the predictions are
perfectly accurate: Their algorithm runs the baseline algorithm
whenever the predicted matching size is smaller than $\beta n$.
Furthermore, their analysis relies on the assumption that the optimal matching size equals~$n$.

Although Choo et al.\ do not explicitly state the assumption that the optimal matching has size $n^* = n$, this assumption underlies several parts of their analysis, including Theorem~4.1 and Lemmas~4.3 and~4.4. For a detailed discussion, see Section~\ref{subsec:comparision-with-Choo et al.}. Consequently, their analysis fails to establish $(1-o(1))$-consistency when the optimal matching size is smaller than~$n$, even if the predicted matching size exceeds $\beta n$. In the more general setting with arbitrary optimal matching size, the analysis in the random arrival order model requires a refined argument: after switching to the baseline algorithm, its performance depends on how many matches are still possible once the sampling phase has passed. Our analysis addresses this using Lemma~\ref{lemma:whp} to bound the expected number of such remaining optimal matches and Corollary~\ref{cor:matching-size-switching-baseline} to bound the algorithm's expected total number of matches.

The paper is structured as follows. Section~\ref{sec:prelim} introduces preliminaries, Section~\ref{sec:related} reviews related work, and Section~\ref{subsec:comparision-with-Choo et al.} compares our result to that by Choo et al.~\cite{online-bipartite-matching-imperfect-advice}.
Section~\ref{sec:algo} presents our algorithm and its analysis: we first outline the algorithm, 
then provide the key components of the analysis in Sections~\ref{subsec:estimated-number-of-matching}--\ref{subsec:succeed-rate}, 
and combine these results to prove our main theorem in Section~\ref{subsec:prove-thm}.
Section~\ref{sec:conc} concludes the paper.

\subsection{Preliminaries}
\label{sec:prelim}
A \emph{matching} in a graph with edge set~$E$ is a set $M\subseteq E$
of edges such that each vertex of the graph has at most one incident edge
in~$M$.
We consider the problem of online unweighted bipartite matching in a graph $G =
(U\cup V, E)$, where $U$ and $V$ are the sets of $n$ offline and $n$ online
vertices, respectively.  The online vertices arrive sequentially. Upon the
arrival of an online vertex, the algorithm must either
match it to one of its available neighbors (i.e., to a neighbor that has
not yet been matched) or leave it unmatched. The objective value is the
number of online vertices that have been matched, also referred to as the
number of matches made by the algorithm.

We focus on the random arrival order model, where the arrival order of the
online vertices $v \in V$ is a uniformly random permutation. The performance
of a (possibly randomized) algorithm is measured 
using the \emph{competitive ratio}, defined as 
\[
    \inf_{G = (U\cup V,E)}  \frac{E_{V\text{'s arrival seq.}} \left[ E_{\text{random choices of alg.}}[\#\text{matches by algorithm}] \right]}{\#\text{matches by optimal solution}}
\]
The size of the optimal matching does not depend on the arrival order.
The size of the matching produced by an algorithm depends on the
order of the arrival sequence and on the random choices made by
the algorithm, hence we take the expectation over both of these.
The competitive ratio is a value $\le 1$, and the closer to $1$
it is, the better.

Learning-augmented algorithms are algorithms that are provided
with (untrusted) predic\-tions, and the goal is to design algorithms that
achieve performance close to the optimal solution when the predictions
are accurate, but still maintain a competitive ratio close to that of
online algorithms without predictions if the predictions are arbitrarily
wrong. These properties are formalized as \emph{consistency} and
\emph{robustness}: An algorithm is \emph{$\gamma$-consistent} if its competitive
ratio is at least $\gamma$ when the predictions are entirely correct,
and it is \emph{$\rho$-robust} if its competitive ratio is at
least $\rho$ no matter how wrong the predictions are.
In addition, one is interested in \emph{smoothness}, i.e., one
aims to show that the competitive ratio of an algorithm degrades
gracefully between consistency and robustness, as a function
of the prediction error.

\begin{figure}[t]
\centering
\begin{subfigure}{0.5\linewidth}
    \includegraphics[width=\linewidth]{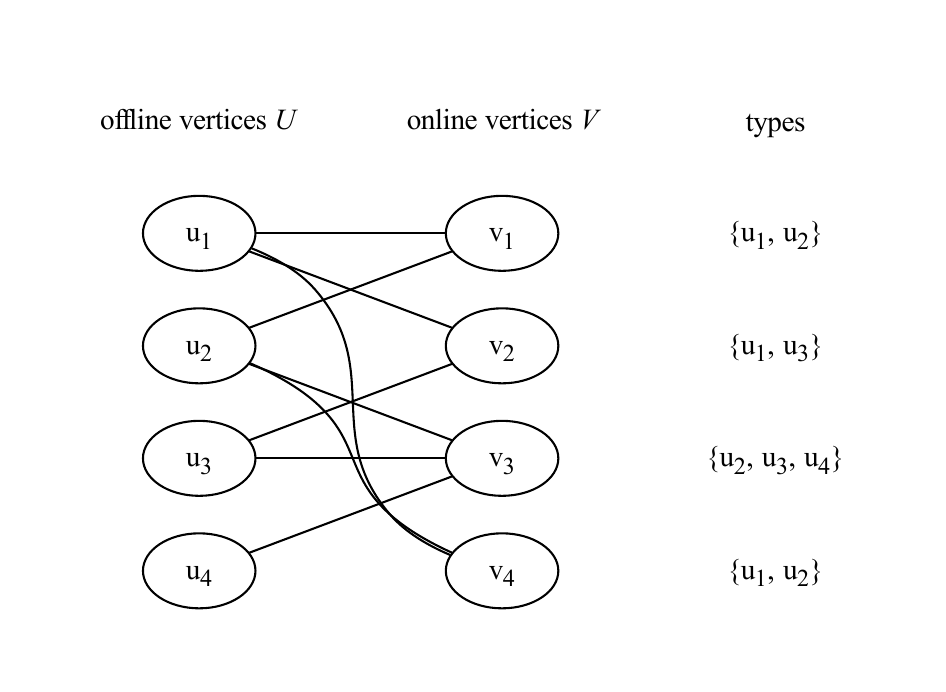} 
    \caption{Bipartite graph and types of online vertices, e.g., $v_1$ and $v_4$ both have type $\{u_1,u_2\}$ as they have edges to $u_1$ and $u_2$}
    \label{fig: example-graph}
\end{subfigure}
\hfill
\begin{subfigure}{0.45\linewidth}
\centering
        \begin{tabular}{c c c}
            \toprule
             & Type $t$ & Count $c^*(t)$\\
            \midrule
                    & $\{u_1, u_2\}$ & 2 \\
            $T^*$   & $\{u_1, u_3\}$ & 1 \\
                    & $\{u_2, u_3, u_4\}$ & 1 \\
            \bottomrule \\ 
            $2^U \setminus T^*$ & \dots & 0
        \end{tabular}
        \caption{Table with number $c^*(t)$  of type $t$ vertices, e.g., $c^*(\{u_1,u_2\}) = 2$}
        \label{table: table-number-of-types}
\end{subfigure}
\caption{Illustration of vertex types}
\label{fig:types}
\end{figure}
The \emph{type} of an online vertex $v \in V$ is the set of
offline vertices in $U$ that are adjacent to~$v$, i.e., the set
$\{u \in U \mid \{u,v\} \in E\}$. There are $2^n$ possible types,
but since there are only $n$ online vertices, the number of different
types that occur in any given arrival sequence is at most~$n$.
The idea of representing online vertices as types was first introduced
by Borodin et al.~\cite{define-online-vertex-as-type}. The set of online vertices
can be represented by a function $c^*$ that maps each type $t$
to the number of type $t$ vertices contained among the online vertices.
We denote by $T^*$ the set of types $t$ with $c^*(t)>0$,
and let $r^*=|T^*|$.
See Figure~\ref{fig:types} for an illustration.
We assume that the algorithm is provided with a function $\hat{c}$
that maps each type $t$ to the number of online vertices of type $t$
that are predicted to arrive, and we denote the set of types $t$
with $\hat{c}(t)>0$ by $\hat{T}$ and its cardinality by~$\hat{r}$.
It is clear that $\sum_{t\in T^*} c^*(t)=n$,
and we assume also that $\sum_{t\in \hat{T}} \hat{c}(t)=n$.
We denote the $L_1$-distance between $c^*$ and $\hat{c}$
by $L_1(c^*,\hat{c})=\sum_t |c^*(t)-\hat{c}(t)|$.
The type prediction $\hat{c}$ determines a predicted
graph $\hat{G}=(U\cup \hat{V},\hat{E})$ in which $\hat{V}$ contains
exactly $\hat{c}(t)$ vertices of type~$t$.
We use $n^*$ to denote the size of an optimal matching in $G$
and $\hat{n}$ to denote the size of an optimal
matching in $\hat{G}$.

If $p$ and $q$ are two probability distributions over a domain~$T$,
their $L_1$-distance $L_1(p,q)$ is defined as
$L_1(p,q)=\sum_{t\in T} |p(t)-q(t)|$.
We can identify the online vertices with a probability distribution
$p$ where type $t$ has probability $c^*(t)/n$, and the predicted
online vertices with a probability distribution $q$ where type~$t$
has probability $\hat{c}(t)/n$.
Note that $L_1(c^*,\hat{c})=n\cdot L_1(p,q)$ holds in this case.
In order to estimate the $L_1$-distance between the probability distributions
representing the predicted online vertices
and the true online vertices based on a sample, we use
the following theorem stated by Choo et al.~\cite{online-bipartite-matching-imperfect-advice},
which in turn is based on results by Jiao et al.~\cite{estimated-l1-distance}.

\begin{theorem}[Choo et al.~\cite{online-bipartite-matching-imperfect-advice}]\label{thm:estimated-L1}
For any $\varepsilon > 0$ and $\delta' \in (0,1)$, let $q$ be a reference distribution over a domain $T$ of size $r$. There exists an even integer $s \in \Theta\left( \frac{r \log(1/\delta')}{\varepsilon^2 \log{(r)}} \right)$ such that an algorithm drawing $s_1 + s_2$ i.i.d samples from an unknown distribution $p$ over $T$, where $s_1, s_2 \sim \text{Poisson}(s/2)$, outputs an estimate $\hat{L}_1(p,q)$ satisfying $|\hat{L}_1(p,q) - L_1(p,q)| \leq \varepsilon$ with probability at least $1-\delta'$.
\end{theorem}

For ease of reference, we list the relevant notation used throughout this paper in one place in Appendix~\ref{app-sec:notation}.

\subsection{Related Work} \label{sec:related}
\subsubsection{Online Bipartite Matching without Predictions}
Karp et al.~\cite{karp1990optimal} first introduced the online bipartite matching problem and presented a randomized \textsc{Ranking} algorithm, which achieves a competitive ratio of $1-1/e$ in the adversarial arrival model. 
Goel and Mehta~\cite{goel2008online} studied the random arrival order model and presented a deterministic algorithm with a competitive ratio of $1-1/e$. They further showed that no online algorithms in the random order model can achieve a competitive ratio better than 0.83.
Karande et al.~\cite{karande2011online} demonstrated that \textsc{Ranking} achieves a competitive ratio of at least 0.653 in the random arrival order model. They also described a family of graphs that show an upper bound for the competitive ratio of \textsc{Ranking} of 0.727 in the random order model. 
Mahdian and Yan~\cite{mahdian2011online} showed that \textsc{Ranking} achieves competitive ratio at least~$0.696$. 

Another model is the online stochastic model, where each online vertex $v \in V$ arrives according to some known probability distribution. We are given a probability distribution $\mathcal{D}$ over the elements of $V$ such that each time a vertex $v$ arrives, it is drawn independently from the distribution $\mathcal{D}$. This model is also referred to as the known i.i.d.\ arrival model. Feldman et al.~\cite{feldman2009online}, who first studied the known i.i.d.\ arrival model, presented a $0.67$-competitive algorithm. 
Manshadi et al.~\cite{manshadi2012online} established a hardness bound of $0.823$ for randomized algorithms in the known i.i.d.\ model and also presented an algorithm achieving a competitive ratio of $0.702$. 
Jaillet and Lu~\cite{jaillet2014online} presented a $0.725$-competitive algorithm, and Brubach et al.~\cite{brubach2016new, brubach2020online} further improved the competitive ratio to $0.7299$.

\subsubsection{Learning-Augmented Algorithms for Matching}
Learning-augmented algorithms have been studied across various fields, including the ski-rental problem~\cite{purohit2018improving, wei2020optimal,gollapudi2019online,anand2020customizing,angelopoulos2024online}, non-clairvoyant scheduling~\cite{purohit2018improving, wei2020optimal}, and caching~\cite{lykouris2021competitive,rohatgi2020near,wei2020better}.
For further information about this area,
see the survey by Mitzenmacher and Vassilvitskii~\cite{mitzenmacher2022algorithms}, and for further references, refer to the website managed by Lindermayr and Megow~\cite{predictions-github}.

For online bipartite matching, various prediction models have been explored. 
Aamand et al.~\cite{aamand2022optimal} studied adversarial arrival models using predicted degree information as advice. Their algorithm is optimal for graphs drawn from the Chung-Lu-Vu random graph model~\cite{chung2003spectra}, where the expected values of offline degrees are known and provided as advice.

Feng et al.~\cite{feng2021two} proposed a two-stage vertex-weighted variant in which the advice is a proposed matching for the online vertices arriving in the first stage.
Jin and Ma~\cite{jin2022online} further established a tight robustness-consistency trade-off for the two-stage model, presenting an algorithm with $(1-(1-\sqrt{1-R})^2)$-consistency and $R$-robustness, where $R \in [0,3/4]$.

Antoniadis et al.~\cite{NEURIPS2020_5a378f84} studied online weighted bipartite matching in a random arrival model. In their model, the predictions provide information about the maximum edge weights adjacent to each offline vertex. Additionally, their analysis is based on a hyperparameter of prediction confidence. Consequently, the consistency and robustness of their algorithm depend on the trade-off hyperparameter. 

Li et al.~\cite{li2023learning} studied online weighted capacity
bipartite matching using pre-trained reinforcement learning as advice. In their
setting, the offline vertices have
capacities, denoted by $c_u$, that limit the number of matches they can make
with online vertices. The setting is referred to as \emph{no-free disposal}. They also
consider \emph{free disposal}, where each offline vertex can match with more than
$c_u$ online vertices, but only the top $c_u$ rewards are counted. Their
algorithm achieves $(1-\rho)$-consistency and $(\rho \cdot \beta)$-robustness,
where $\rho$ is the hedging ratio that balances robustness and consistency, and $\beta$ is the
best competitive ratio of online algorithms without predictions in their setting.

As mentioned earlier, Choo et al.~\cite{online-bipartite-matching-imperfect-advice} studied unweighted bipartite matching
with random-order arrivals. The predictions specify the types of the online vertices, and the number of online vertices of
each type.
In this paper, we modify their algorithm and its analysis to address several limitations.
A detailed comparison of our result with their work is provided in the following section.

\subsection{Comparison with the Result of Choo et al.~\texorpdfstring{\cite{online-bipartite-matching-imperfect-advice}}{[9]}}
\label{subsec:comparision-with-Choo et al.}
In this section, we discuss our interpretation of the analysis of Choo et al.~\cite{online-bipartite-matching-imperfect-advice} and clarify how our results differ. Their main result is the following:

\smallskip

\textbf{Theorem~4.1 (in Choo et al.~\cite{online-bipartite-matching-imperfect-advice})}
\textit{
For any advice $\hat{c}$ with $|\hat{T}| = \hat{r}, \varepsilon > 0$ and $\delta > \frac{1}{poly(\hat{r})}$, let $\hat{L}_1$ be the estimate of $L_1(p,q)$ obtained from $k = s_{\hat{r},\varepsilon,\delta} \cdot \sqrt{\log(\hat{r} + 1)}$ IID samples of $p$. \textsc{TestAndMatch} produces a matching of size $m$ with competitive ratio of at least $\frac{\hat{n}}{n} - \frac{L_1(p,q)}{2} \ge \beta$ when $\hat{L}_1 \le 2(\frac{\hat{n}}{n} - \beta) - \varepsilon$, and at least $\beta \cdot (1-\frac{k}{n})$ otherwise, with success probability $1-\delta$.
}

\smallskip

Although their Theorem 4.1
does not state it explicitly, their proof relies on the assumption $n^* = n$,
and so the theorem is only proven for that case (see below for details).
Hence, the theorem shows only for $k=o(n)$ and $n^*=n$ that their algorithm achieves $1$-consistency and
$(\beta-o(1))$-robustness with probability $1-\delta$.
They mention that the guarantees on the competitive ratio that are obtained with probability $1-\delta$ can be converted into guarantees on the expected competitive
ratio, giving $(1-\delta)$-consistency and $\beta(1-o(1))(1-\delta)$-robustness in expectation (and they suggest $\delta=0.001$ as a possible choice). We state our results
directly in terms of the expected competitive ratio instead.

\subsubsection{Differences to Our Result} 
The algorithm \textsc{TestAndMatch} by Choo et al.\ runs the $\beta$-competitive baseline algorithm on all arrivals whenever $\frac{\hat{n}}{n} \le \beta$.
In contrast, our algorithm replaces this requirement by $\frac{\hat{n}}{n} \le \alpha$ for an arbitrarily small constant $\alpha>0$. As a result, our algorithm obtains both $(1-o(1))$-consistency and competitive ratio greater than $\beta$ over a larger set of inputs, as we explain in the following.

First, consider consistency. Assume that the optimal matching has size $n^*$ with $\alpha n \le n^* < \beta n$ and that the predictions are accurate: $\hat{n} = n^*$ and $L_1({c^*,\hat{c}}) = 0$. Our algorithm is $(1-o(1))$-competitive for such inputs.
The algorithm of Choo et al.~\cite{online-bipartite-matching-imperfect-advice}, however, does not achieve $(1-o(1))$-consistency because the condition $\hat{n} = n^{*} < \beta n$ causes it to run the baseline algorithm on all arrivals, yielding
only competitive ratio~$\beta$.
Choo et al.\ mention the possibility of modifying the predictions by
adding a complete bipartite graph between unmatched offline and unmatched predicted vertices (see Section~5.3 and Appendix~D.3 of~\cite{online-bipartite-matching-imperfect-advice}), thus ensuring that the maximum predicted matching has size $n$ for the modified predictions $\hat{c}'$. Note that this causes the prediction error to increase
from $0$ to $L_1(c^*,\hat{c}') = 2(n-\hat{n}) > 2(1-\beta)n$,
which implies $\hat{L}_1(p,q) > 2(1-\beta) - \varepsilon$ with high probability, exceeding the threshold.
Consequently, their algorithm still switches to the baseline algorithm after the sampling
phase, achieving only competitive ratio $\beta - o(1)$. 

Second, consider inputs with non-zero prediction error.
The analysis by Choo et al.\ establishes competitive ratio better than
$\beta$ only if $n^*=n$ and $\hat{L}_1(p,q)<2(\frac{\hat{n}}{n}-\beta)-\varepsilon$. On the other hand,
our analysis establishes competitive ratio better than $\beta$
whenever $\hat{n}\ge \alpha n$ and $\hat{L}_1(p,q)<\frac{2\hat{n}}{n}\cdot\frac{1-\beta}{1+\beta}-\varepsilon$. In particular, we obtain competitive ratio better than $\beta$
for all inputs where both $n^*$ and $\hat{n}$ lie in $[\alpha n, n]$ and
the prediction error is small.

\subsubsection{Implicit Perfect Matching Assumption in Choo et al.'s Analysis}
The analysis of Choo et al.~\cite{online-bipartite-matching-imperfect-advice} implicitly assumes $n^* = n$. 
One of their key lemmas \cite[Lemma 4.3]{online-bipartite-matching-imperfect-advice}
states that, if an algorithm makes $j$ matches during the first $k$
arrivals and then switches to the baseline algorithm, the expected
total number of matches is at least $j+\beta(n-k-j)$. The proof
uses the argument that the number of matches that are still possible during the
remaining $n-k$ arrivals is at least $n-k-j$, and this argument clearly holds only
under the assumption that the size of the optimal matching is equal to~$n$.
Lemma~4.4 in \cite{online-bipartite-matching-imperfect-advice} then builds directly on Lemma~4.3, and their proof of Theorem 4.1 in turn relies on Lemma 4.4. Thus, their proofs of all these statements require the assumption that $n^* = n$.

\subsubsection{Adjustment of Failure Probability}
Choo et al.~\cite{online-bipartite-matching-imperfect-advice} use a failure probability $\delta > 1/\poly(\hat{r})$, where $\hat{r}$ is the number of distinct predicted types. As $\hat{r}$ can be very small even if $\hat{n}=n$,
it is not guaranteed that $1/\poly(\hat{r})$ is $o(1)$.
We instead work with a failure probability of $\delta > 1/\poly(n+1)$, allowing us to ensure $\delta = o(1)$ even when $\hat{r}$ is small.

\section{Learning-Augmented Algorithm and Analysis}
\label{sec:algo}

In this section, we present a learning-augmented algorithm that is $(1-o(1))$-consistent and $(\beta - o(1))$-robust provided
that $\hat{n}\ge \alpha n$ for an arbitrary constant $\alpha\in(0,1]$, where $\beta$ is the best-known competitive ratio for online bipartite matching in the random arrival order model without predictions, currently $\beta=0.696$~\cite{mahdian2011online}.
Additionally, we provide a smoothness analysis for our algorithm. 

The high-level idea of the algorithm by Choo et al.~\cite{online-bipartite-matching-imperfect-advice} is to use an initial prefix of the arrival sequence as a sample to assess the predictions and, based on this assessment, either continue with the prediction-based algorithm or switch to the online bipartite matching algorithm without predictions.

\begin{algorithm}
\caption{\textsc{Test-and-Match+}}
\label{alg:LA}

\textbf{Input:} Set of predicted input types $\hat{c}$ with $\hat{r} = |\hat{T}|$, desired error bounds $\delta'>0$ and $\varepsilon>0$

\textbf{Initialization:} $i \leftarrow 0$, $T^s_p \leftarrow \emptyset$, $A \leftarrow \emptyset$,
$\tau \gets \frac{2\hat{n}}{n}\cdot \frac{(1 - \beta)}{(1+\beta)}$
\begin{algorithmic}[1]
\If{$\frac{\hat{n}}{n} < \alpha$}
    \State Run \textsc{Baseline} for the whole input.
\EndIf

\State Define $s_{n,\varepsilon,\delta'} = \Theta\left(\frac{(n+1) \log(1/\delta')}{\varepsilon^2 \log(n+1)}\right)$
\State Draw the numbers $s_1, s_2 \sim \text{Poisson}(\frac12 s_{n,\varepsilon,\delta'})$

\If{$s_1 + s_2 > s_{n,\varepsilon,\delta'} \cdot \left( 1 + \sqrt{\log(n + 1)} \right)$}
    \State Run \textsc{Baseline} for all remaining input \Comment{Happens with probability $\delta_{poi} = \mathcal{O} \left( \frac{1}{\poly(n+1)}\right)$}
\EndIf

\State Compute a maximum matching $\hat{M}$ using $\hat{c}$ as input

\While{the size of $T^s_p$ is smaller than $s_1+s_2$}\Comment{Sample $s_1+s_2$ input vertices with replacement}
\State Flip a coin that lands on heads with probability $i/n$ and tails otherwise
\If{the coin flip yields heads}
    \State Pick an element uniformly at random from the set $\{A[0],\dots, A[i-1]\}$ and add it to $T^s_p$
\Else
    \State Add the $(i+1)$-th arriving online vertex to $T^s_p$ and store it in $A[i]$
    \State $i \gets i+1$ 
    \State Run \textsc{Mimic} (Algorithm~\ref{alg:mimic}) to process the online vertex
\EndIf
\EndWhile
\State Run the algorithm of Theorem~\ref{thm:estimated-L1} on the distribution $q = \frac{\hat{c}}{n}$ and the sample $T^s_p$ from distribution $p=\frac{c^*}{n}$ with domain size $n$ to obtain an estimate $\hat{L}_1(p,q)$ such that $|\hat{L}_1(p,q) - L_1(p,q)| \leq \varepsilon$ with probability $1-\delta'$

\If{$\hat{L}_1(p,q) \leq \tau - \varepsilon$}
    \State Run \textsc{Mimic} (Algorithm~\ref{alg:mimic}) on the remaining input
\Else
    \State Run \textsc{Baseline} on the remaining input
\EndIf
\end{algorithmic}
\end{algorithm}
\begin{algorithm}
\textbf{Input:} Let $c$ be a copy of $\hat{c}$ (initialized before the first online vertex arrives) 
\caption{\textsc{Mimic}}
    \label{alg:mimic}
    \begin{algorithmic}[1]
            \State When an input vertex $v$ with type $t$ arrives:
            \If {$c(t)> 0$}
                \State Match $v$ according to an arbitrary unused type $t$ match in $\hat{M}$ if such a match exists
                \State Decrease $c(t)$ by $1$
            \Else
                \State Leave $v$ unmatched
            \EndIf
    \end{algorithmic}
\end{algorithm}

Our algorithm (see Algorithm~\ref{alg:LA}) follows the approach
proposed by Choo et al.~\cite{online-bipartite-matching-imperfect-advice} of using a prefix of the
arrival sequence as a sample, estimating the $L_1$-distance between the predicted and true
distribution of input vertices via Theorem~\ref{thm:estimated-L1}, and then either following
the predictions (if the estimated $L_1$-distance is
at most $\tau-\varepsilon$ for $\tau=\frac{2\hat{n}}{n}\cdot \frac{1 - \beta}{1+\beta}$)
or switching to the baseline algorithm for the remaining input. Here, the baseline algorithm
is chosen as the best-known algorithm for online bipartite matching in the random arrival
order model without predictions, and its competitive ratio is denoted by~$\beta$.
We use \textsc{Baseline} to refer to that algorithm. The algorithm that follows
the predictions is called \textsc{Mimic} (see Algorithm~\ref{alg:mimic}): Whenever an online vertex arrives,
the algorithm matches it according to a pre-computed maximum matching $\hat{M}$ of the predicted
input graph~$\hat{G}$. During the sampling phase, the arriving vertices are also processed
by \textsc{Mimic}.
More precisely, our algorithm evaluates
the accuracy of the predictions by comparing the predicted distribution ($q = \hat{c}/n$) with the actual input distribution ($p = c^*/n$)
based on a sample with expected size $s_{n,\varepsilon,\delta'} = \Theta \left(\frac{(n+1) \log(1/\delta')}{\varepsilon^2 \log(n+1)}\right)$.
As done by Choo et al.~\cite{online-bipartite-matching-imperfect-advice}, the actual sample size is chosen as $s_1+s_2$, where
$s_1$ and $s_2$ are drawn independently from a Poisson distribution with parameter $\frac12 s_{n,\varepsilon,\delta'}$.
If the sample size is too large, i.e., if $s_1+s_2>s_{n,\varepsilon,\delta'}(1+\sqrt{\log(n+1)})$, which happens with probability~$o(1)$, we
do not sample and run \textsc{Baseline} from the beginning.
We refer to $s_{n,\varepsilon,\delta'}(1+\sqrt{\log(n+1)})$ as the \emph{sample size limit}.
Otherwise, we compute a sample $T^s_p$ of $s=s_1+s_2$ input vertices with replacement
(lines 10--19 of Algorithm~\ref{alg:LA} implement sampling with replacement in the same way as done by
Choo et al.~\cite{online-bipartite-matching-imperfect-advice})
and use the algorithm of Theorem~\ref{thm:estimated-L1} to determine an estimate $\hat{L}_1(p,q)$ of the $L_1$-distance $L_1(p,q)$ between $p$ and $q$. 

We apply that algorithm using domain size $n+1$ (as opposed to $\hat{r}+1$ used by Choo et al.~\cite{online-bipartite-matching-imperfect-advice}) using
a similar approach as Choo et al.~\cite{online-bipartite-matching-imperfect-advice}: We consider only the $\hat{r}\le n$ types $t$ with $\hat{c}(t)>0$,
a single dummy type $t'$ that represents all types that appear in the true arrival sequence but not in the predicted arrival sequence,
and $n-\hat{r}$ arbitrary other types.
If $\hat{L}_1(p,q)$ is smaller than the threshold $\tau-\varepsilon$ for
$\tau=\frac{2\hat{n}}{n}\cdot \frac{(1 - \beta)}{(1+\beta)}$, we consider the prediction error small and continue to
run \textsc{Mimic} for the remainder of the input. Otherwise, we run \textsc{Baseline} on the remainder of the input
(while disregarding edges of newly arriving online vertices to offline vertices that have already been matched during
the sampling phase).

The main differences between our algorithm \textsc{Test-and-Match+} and the original
algorithm \textsc{Test-and-Match} are: Our algorithm only requires $\hat{n}\ge \alpha n$
for an arbitrary constant $\alpha\in(0,1]$, as opposed to the condition that $\hat{n}\ge \beta n$
for $\beta=0.696$ required by Choo et al.~\cite{online-bipartite-matching-imperfect-advice}. We choose the expected sample size as
$s_{n,\varepsilon,\delta'} = \Theta\left(\frac{(n+1) \log(1/\delta')}{\varepsilon^2 \log(n+1)}\right)$
as opposed to $\Theta\left(\frac{(\hat{r}+1)\log(1/\delta')}{\varepsilon^2 \log(\hat{r}+1)}\right)$,
and multiply by $1+\sqrt{\log(n+1)}$ instead of $\sqrt{\log(\hat{r}+1)}$ to determine the
sample size limit. This has the advantage that the probability
that the sample size exceeds the sample size limit can be bounded by $1/\poly(n)$ instead of $1/\poly(\hat{r})$
and is thus guaranteed to be~$o(1)$ even if $\hat{r}$ is a small constant.
Furthermore, the threshold on the estimated $L_1$-distance below which we run \textsc{Mimic} on the remaining input is set to
$\frac{2\hat{n}}{n}\cdot \frac{1 - \beta}{1+\beta}-\varepsilon$,
as opposed to $2(\frac{\hat{n}}{n}-\beta)-\varepsilon$ by Choo et al.~\cite{online-bipartite-matching-imperfect-advice}.
This makes the threshold useful also for situations where $n^*<n$.
The main changes in the analysis include: Instead of assuming $n^*=n$, we need
to bound $n^*$ based on $\hat{n}$ and the prediction error. The analysis of the
expected size of an optimal matching on the remaining input becomes significantly
more challenging without the assumption that $n^*=n$.

Our main result about the performance of the \textsc{Test-and-Match+} algorithm (executed with $\delta' = \frac{1}{\log \log \log n}$) can be stated as follows:

\begin{theorem}
\label{thm:competitive-ratio}
Let $\alpha$ be a constant in the range $(0,1]$, $\varepsilon \in (0,\alpha \frac{1-\beta}{1+\beta}]$ be
another constant.
Consider any instance with predicted input $\hat{c}$ and predicted matching size $\hat{n} \geq \alpha n$,
and let $L_1(p,q)$ be the prediction error.
If $L_1(p,q) \leq \frac{2\hat{n}}{n} \cdot \frac{(1 - \beta)}{(1+\beta)} - 2\varepsilon$, \textsc{Test-and-Match+} achieves competitive ratio at least
$1 - \frac{2 L_1(p,q)}{2 \alpha + L_1(p,q)} - o(1)$ (and at least $\beta - o(1)$). Otherwise, the algorithm achieves competitive ratio at least $\beta - o(1)$.
\end{theorem}

Note that if the predicted matching size $\hat{n}$ is less than $\alpha n$, we run the \textsc{Baseline} algorithm right away, which guarantees a competitive ratio of at least $\beta$. This ensures the robustness of the solution.

The following three subsections provide the building blocks for the proof of
Theorem~\ref{thm:competitive-ratio}. 
In Section~\ref{subsec:estimated-number-of-matching}, we bound the number of matches achieved when running \textsc{Mimic} for the entire input as well as the number of matches obtained if we switch to \textsc{Baseline} after we obtain the estimate $\hat{L}_1(p,q)$. We also provide an upper bound on the size of the optimal matching.
In Section~\ref{subsec:thresold}, we justify the choice of the threshold $\tau$ to ensure that our algorithm achieves a competitive ratio better than $\beta$ if we continue with \textsc{Mimic}.
In Section~\ref{subsec:succeed-rate}, we bound the probability that the sampling fails (because the sample size $s_1+s_2$ exceeds the limit we set)
or that the estimate returned when calling the algorithm from Theorem~\ref{thm:estimated-L1} has additive error greater than~$\varepsilon$.
Finally, in Section~\ref{subsec:prove-thm}, we combine these building blocks to complete the proof of Theorem~\ref{thm:competitive-ratio}.

\subsection{Bounds on the Number of Matches}
\label{subsec:estimated-number-of-matching}
The following two lemmas establish
bounds on the size of the matching produced by \textsc{Mimic} and on the size of the optimal matching in terms of $\hat{n}$ and the prediction error.
In the proofs of these lemmas, we use the following terminology:
If a vertex of type $t$ arrives, and if that vertex is the $i$-th vertex of type $t$
that arrives, we say that the vertex was \emph{predicted} if $i\le \hat{c}(t)$
and \emph{unpredicted} if $i> \hat{c}(t)$.

\begin{lemma}
\label{lemma:mimic-estimated-matching-size}
Assume that we run \textsc{Mimic} for the entire input sequence. Then the number of matches we obtain is at least $\hat{n} - \frac{n}{2} L_1(p,q)$,
where $L_1(p,q)$ is the $L_1$-distance between the input distribution $p=c^*/n$ and the predicted distribution $q = \hat{c}/n$.
\end{lemma}
\begin{proof}
For each unpredicted vertex that arrives, we lose at most one match compared
to the optimal matching in $\hat{G}$. The number of unpredicted vertices
is double counted in $L_1(\hat{c},c^*)$, because for each unpredicted vertex of type $t$,
the count of some predicted vertex type is reduced by one and the count
of $t$ is increased by one. Therefore, the number of unpredicted vertices
is equal to $\frac12 L_1(\hat{c},c^*)=\frac{n}{2} L_1(p,q)$.
\end{proof}

\begin{lemma} \label{lemma:estimated-opt}
    The size $n^*$ of the optimal matching is at most $\hat{n} + \frac{n}{2} L_1(p,q)$.
\end{lemma}
\begin{proof}
Consider the optimal matching $M^*$.
Let $n^*_1$ be the number of predicted vertices matched in $M^*$
and $n^*_2$ be the number of unpredicted vertices matched in $M^*$.
Clearly, $n^*_1\le \hat{n}$, as $\hat{n}$ is the size of a maximum matching in
$\hat{G}$.
Furthermore $n^*_2$ is bounded by the number of unpredicted vertices, which is
equal to $\frac{n}{2}L_1(p,q)$. As $n^*=n^*_1+n^*_2$, the lemma follows.
\end{proof}

Furthermore, if we run \textsc{Baseline} instead of \textsc{Mimic} after we determine $\hat{L}_1(p,q)$, the following lemma gives the expected number of optimal matches in the remaining input. Afterwards, we use this to give a lower bound on the expected total number of matches we obtain after we switch to run \textsc{Baseline} for the remaining input.

\begin{lemma}\label{lemma:whp}
 Let $\alpha$ be an arbitrary constant in the range $(0,1]$,
 $\varepsilon \in (0,\alpha \frac{1-\beta}{1+\beta}]$ be another constant,
 $\delta' = \frac{1}{\log\log\log n}$, and sample size $k = s_1 + s_2 \leq s_{n,\varepsilon,\delta'} \cdot \left( 1 + \sqrt{\log(n+1}) \right)$.
 Let $k'\le k$ be the number of vertex arrivals during the sampling phase.
Then the number of matches made by the optimal matching among the remaining $n-k' \ge n-k$ arrivals is at least $\left(\frac{n-k}{n} - o(1)\right) \cdot n^*$ with probability $1 - o(1)$.
\end{lemma}

\begin{proof}
Recall that $s_{n,\varepsilon,\delta'} = \Theta\left(\frac{(n+1)\log(1/\delta')}{\varepsilon^2 \log(n+1)}\right)$.
Fix an arbitrary optimal matching~$M^*$ of size $n^*$.
The algorithm samples $k$ vertices with replacement, so the number of vertex arrivals during the
sampling phase is some $k'\le k$, and the number of remaining vertices is $n-k'\ge n-k$. 
Let $O_1$ be the number of vertices sampled in the sampling phase that are matched
in $M^*$, and let $O_2$ be the number of vertices among the remaining vertices that
are matched in $M^*$. We aim to show that
$O_2 \ge \left(\frac{n-k}{n} - o(1)\right) \cdot n^*=\frac{n-k}{n}n^*-o(n^*)$ with probability $1 - o(1)$.
Note that $O_2 \ge n^* - O_1$.
Therefore, it suffices to show that
$O_1 > \frac{k}{n} n^* + o(n^*)$ happens only with probability $o(1)$.

The sampling phase samples $k$ online vertices with replacement.
Each sampled vertex has a probability of $n^*/n$ to be a vertex matched in~$M^*$.

Let $X_i$ for $i = 1,\ldots, k$ be independent Bernoulli random variables, where $X_i = 1$ if the $i$-th sampled vertex
is matched in $M^*$, and $X_i = 0$ otherwise.
Then $O_1= \sum_{i=1}^{k} X_i$, and
$E[O_1] = E[\sum_{i=1}^{k} X_i ] = \sum_{i=1}^k E[X_i] = \frac{k}{n}n^*$.
We can assume that $k=s_{n,\varepsilon,\delta'} \cdot \left( 1 + \sqrt{\log(n+1}) \right)$
in the remainder of the proof as $O_1$ is maximized (and $O_2$ minimized), in a stochastic
sense, when the sampling phase is as long as possible.
Thus $E[O_1]=\frac{k}{n}n^* = \Theta\left(\frac{(n+1)\log(1/\delta')}{\varepsilon^2 \log(n+1)}\right)\cdot (1+\sqrt{\log(n+1)})\cdot\frac{n^*}{n}
= \Theta\left(\frac{\log(1/\delta')}{\varepsilon^2 \log(n+1)}\right)\cdot (1+\sqrt{\log(n+1)})\cdot n^*$.

We consider two cases based on the value of~$n^*$:

{\bf Case 1:} $n^* \geq \frac{\sqrt{\log (n+1)}}{\log\log\log n}$

We have: \begin{eqnarray*}
E[O_1] & \ge & \Theta\left(\frac{\log(1/\delta')}{\varepsilon^2 \log(n+1)}\right)\cdot (1+\sqrt{\log(n+1)})\cdot \frac{\sqrt{\log (n+1)}}{\log\log\log n}\\
&\geq& \Theta\left(\frac{(1+\beta)^2 \log\log\log\log n}{{\alpha}^2 (1-\beta)^2 \log\log\log n} \left(1 + \frac{1}{\sqrt{\log(n+1)}} \right)\right) \\
&\geq& \frac{1}{\mathcal{O}(\log\log\log n)}
\end{eqnarray*}

A well-known Chernoff bound is
$\text{Pr}[X \geq (1+\delta)E[X]] \leq \left(\frac{e^\delta}{(1+\delta)^{(1+\delta)}} \right)^{E[X]} \leq 2^{-(1+\delta)E[X]}$ when $\delta > 2e - 1$
\cite[Exercise 4.1]{Motwani_Raghavan_1995}. Applying this bound to $O_1$, we get:
\begin{eqnarray*}
\text{Pr}[O_1 \geq (1+\log\log n)E[O_1]] &\leq& 2^{-(1+\log\log n)E[X]}\\
&\leq& 2^{-\frac{1+\log\log n}{\mathcal{O}(\log\log\log n)}} = o(1)
\end{eqnarray*}
As $(\log\log n)E[O_1]
= (\log\log n)\cdot \Theta\left(\frac{\log(1/\delta')}{\varepsilon^2 \log(n+1)}\right)\cdot (1+\sqrt{\log(n+1)})\cdot n^*\\
= \Theta(\frac{(\log\log n)(\log\log\log\log n)\sqrt{\log(n+1)}}{\log(n+1)}n^*)=o(n^*)$,
it follows that $\text{Pr}[O_2\ge \frac{n-k}{n}n^* - o(n^*)]=1-o(1)$, as desired. 

{\bf Case 2:} $n^* < \frac{\sqrt{\log (n+1)}}{\log\log\log n}$

We show that the probability that at least one element in the sample of size $k$
is matched in $M^*$ is $o(1)$. 
\begin{eqnarray*}
    \text{Pr}[O_1 \geq 1] &=& 1 - \text{Pr}[O_1 = 0]\\
    &=& 1 - \left( 1 - \frac{n^*}{n} \right)^{\Theta\left(\frac{(n+1) \cdot \log(1/\delta')}{\varepsilon^2} \left(\frac{1}{\log(n+1)} + \frac{1}{\sqrt{\log(n+1)}} \right)\right)}\\
    &\leq& 1 - \left( 1 - \frac{\sqrt{\log (n+1)}}{(\log\log\log n) \cdot n} \right)^{\Theta\left(\frac{(n+1) \cdot \log(1/\delta')}{\varepsilon^2} \left(\frac{1}{\log(n+1)} + \frac{1}{\sqrt{\log(n+1)}} \right)\right)}
\end{eqnarray*}

\sloppypar{Recall the following version of Bernoulli's inequality: $(1+x)^{r}\geq 1+rx$ for every real number $r\geq 1$ and $x\geq -1$. For sufficiently large $n$, observe that setting $x = - \frac{\sqrt{\log (n+1)}}{(\log\log\log n) \cdot n}$ and setting $r = \Theta\left(\frac{(n+1) \cdot \log(1/\delta')}{\varepsilon^2} \left(\frac{1}{\log(n+1)} + \frac{1}{\sqrt{\log(n+1)}} \right)\right)$ satisfy the required conditions (recalling also
that $\varepsilon$ is a constant and $\delta'=\frac{1}{\log\log\log n}$). Thus, we use the inequality to continue the calculation as follows.
}
\begin{eqnarray*}
     \lefteqn{1 - \left( 1 - \frac{\sqrt{\log (n+1)}}{(\log\log\log n) \cdot n} \right)^{\Theta\left(\frac{(n+1) \cdot \log(1/\delta')}{\varepsilon^2} \left(\frac{1}{\log(n+1)} + \frac{1}{\sqrt{\log(n+1)}} \right)\right)}}\\
     &\leq& 1 - \left(1 - \Theta\left(\frac{\log(1/\delta')}{\varepsilon^2 \log\log\log n} \left(1+\frac{1}{\sqrt{\log(n+1)}} \right)\right)\right) \\
     &=& \Theta\left(\frac{\log(1/\delta')}{\varepsilon^2 \log\log\log n} \left(1+\frac{1}{\sqrt{\log(n+1)}} \right)\right)\\
     &=& \Theta\left(\frac{\log\log\log\log n}{\log\log\log n} \left(1 + \frac{1}{\sqrt{\log(n+1)}} \right)\right)\\
    &=& o(1)
\end{eqnarray*}
Therefore, $\text{Pr}[O_2 \ge \frac{n-k}{n} \cdot n^*] \geq 1 - o(1)$.
\end{proof}

\begin{corollary}
\label{cor:matching-size-switching-baseline}
    Assume that the algorithm has matched $j$ vertices during the sampling phase
    and then switches to \textsc{Baseline} for the remaining input.
    The total expected number of matches made by the algorithm is
    at least $\beta \frac{n-k}{n} \cdot n^* - o(n^*)$,
    where $\beta$ is the competitive ratio of \textsc{Baseline} and $k$ is the number of arrivals before we switch to \textsc{Baseline}.
\end{corollary}
\begin{proof}
By Lemma~\ref{lemma:whp}, with probability $1-o(1)$
the remaining input contains a set $V^*$
of at least $\frac{n-k}{n} n^* - o(n^*)$ vertices that are matched
by the optimal matching.
If the algorithm has already made $j$ matches at the point of
switching to \textsc{Baseline}, $j$~of these $\frac{n-k}{n} n^* - o(n^*)$
vertices in the remaining input may no longer be matchable.
At least $\frac{n-k}{n} n^* - j - o(n^*)$ vertices will still
be matchable, however, and therefore \textsc{Baseline} will
make at least $\beta\left(\frac{n-k}{n} n^* - j - o(n^*)\right)$ matches.
Thus, the total expected number of matches made by the algorithm
is at least $j + (1-o(1))\beta\left(\frac{n-k}{n} n^* - j - o(n^*)\right)
\ge \beta \cdot \frac{n-k}{n} \cdot n^* - o(n^*)$.

\end{proof}

\subsection{Choice of Threshold}
\label{subsec:thresold}
Corollary~\ref{cor:matching-size-switching-baseline} shows that the competitive
ratio of our algorithm is at least $\beta-o(1)$ if the sampling phase
processes $k=o(n)$ vertices and we switch to \textsc{Baseline} for the remainder
of the input. In this section, we justify the choice of the threshold
$\tau-\varepsilon$ to ensure that the algorithm is also
$(\beta-o(1))$-competitive if the algorithm decides to run \textsc{Mimic} on
the whole input.

\begin{lemma}
\label{lemma:set-threshold}
    If $L_1(p,q) \leq \tau$, where $\tau =  \frac{2\hat{n}}{n} \cdot \frac{(1 - \beta)}{(1+\beta)}$ and our algorithm runs \textsc{Mimic} on the whole input, then the competitive ratio is at least $\beta$.
\end{lemma}
\begin{proof}
If we run \textsc{Mimic} on the whole input,
we obtain at least $\hat{n} - \frac{n}{2} L_1(p,q)$ matches by Lemma~\ref{lemma:mimic-estimated-matching-size}.
By Lemma~\ref{lemma:estimated-opt}, the optimal matching has size at most $\hat{n} + \frac{n}{2} (L_1(p,q))$.
Thus, the competitive ratio is at least $\frac{\hat{n} - \frac{n}{2} L_1(p,q) }{\hat{n} + \frac{n}{2} L_1(p,q)}$.
It is easy to show that this ratio is at least $\beta$ provided that
$L_1(p,q)\le \frac{2\hat{n}}{n} \cdot \frac{(1 - \beta)}{(1+\beta)}$.
\end{proof}

If the estimated $L_1$-distance $\hat{L}(p,q)$ is at most
$\tau-\varepsilon$ and the estimation error is at most $\varepsilon$,
it follows that $L(p,q)\le \tau$. Thus, we obtain the following corollary.

\begin{corollary}
If $\hat{L}_1(p,q) \leq \tau - \varepsilon$ and $\hat{L}_1(p,q)$ satisfies the property $|\hat{L}_1(p,q)-L_1(p,q)| \leq \varepsilon$, then our algorithm
runs \textsc{Mimic} on the whole input and has competitive ratio at least $\beta$.
\end{corollary}

\subsection{Bounding the Probability of Sampling Failures}
\label{subsec:succeed-rate}

In Algorithm~\ref{alg:LA}, two cases may prevent the algorithm from carrying out
the sampling phase and obtaining a good estimate of $L_1(p,q)$: First, the sum $s_1+s_2$ of two numbers drawn from a Poisson distribution might
exceed $s_{n,\varepsilon,\delta'} \cdot \left( 1 + \sqrt{\log(n + 1)} \right)$. In this case, we run \textsc{Baseline} for the whole input.
Second, the estimated value of $L_1(p,q)$, obtained via the algorithm of Theorem~\ref{thm:estimated-L1}, may not satisfy
the property $|\hat{L}_1(p,q)-L_1(p,q)| \leq \varepsilon$. This can cause the algorithm to make an incorrect decision regarding
whether to run \textsc{Mimic} or \textsc{Baseline} on the remaining input.
We show that the probability for these cases to happen is $o(1)$, thus reducing the competitive ratio only by $o(1)$.
First, the following lemma, which can be shown by applying standard tail bounds of the Poisson distribution,
bounds the probability that the sample size $s_1+s_2$ is too large.

\begin{lemma}
\label{lemma:probability-sim-poisson}
    Let $\alpha$ be any constant in the range $(0,1]$,
    $\varepsilon \in (0,\alpha \frac{1-\beta}{1+\beta}]$,
    $\delta' = \frac{1}{\log\log\log n}$, and $s_{n,\varepsilon,\delta'} = \Theta \left(\frac{(n+1) \cdot \log(1/\delta')}{\varepsilon^2 \cdot \log(n+1)}\right)$. 
    The probability that two numbers $s_1,s_2$ drawn independently from a
    Poisson distribution with parameter $s_{n,\varepsilon,\delta'}/2$ satisfy
    $s_1 + s_2> s_{n,\varepsilon,\delta'} \cdot \left( 1 + \sqrt{\log(n + 1)} \right)$ is at most
    $\delta_{poi} = \mathcal{O} \left( \frac{1}{\poly(n+1)} \right)$.
\end{lemma}
\begin{proof}
    For a Poisson distribution with parameter~$m>0$, and for any $x > 0$, the following tail bound holds (see, e.g., \cite{Canonne2019}):
    \begin{equation}
    \label{eq:poisson-tail-bound}
        \text{Pr}[|X-m| \geq x] \leq 2 \cdot \exp \left(-\frac{x^2}{2(m+x)} \right)
    \end{equation}
    
    As per Theorem~\ref{thm:estimated-L1}, $s_1$ and $s_2$ are drawn from a Poission
    distribution with parameter $s_{n,\varepsilon,\delta'}/2$ where
    $s_{n,\varepsilon,\delta'} = \Theta \left( \frac{(n+1) \cdot \log(1/\delta')}{\varepsilon^2 \cdot \log(n+1)} \right)$.
    The sample size limit is set to $s_{n,\varepsilon,\delta'}(1+\sqrt{\log(n+1)})$.
    If $s_1+s_2$ exceeds the sample size limit, then at least one of the two
    values must exceed half the sample size limit. Using (\ref{eq:poisson-tail-bound}), we can bound the probability $p_1$ that $s_1$
    exceeds half the sample size limit as follows:
\begin{eqnarray*}
        p_1&=&\Pr \left[ s_1 > \frac12 s_{n,\varepsilon,\delta'} (1 + \sqrt{\log(n+1))} \right] \\
        &\le& \Pr \left[ |s_1 - s_{n,\varepsilon,\delta'}/2| > \frac12 s_{n,\varepsilon,\delta'} \sqrt{\log(n+1)} \right] \\
        &\leq& 2 \cdot \mathbf{exp} \left({- \frac{\frac14 s_{n,\varepsilon,\delta'}^2 \log(n+1)}{2 \left( \frac12 s_{n,\varepsilon,\delta'}+\frac12 s_{n,\varepsilon,\delta'}\sqrt{\log(n+1)} \right)}} \right) \\
        &=& \mathcal{O}\left( (n+1) ^{- \frac{\frac14 s_{n,\varepsilon,\delta'}}{1+\sqrt{\log(n+1)} }} \right) 
\end{eqnarray*}

As $s_{n,\varepsilon,\delta'} = \Theta\left(\frac{(n+1) \cdot \log(1/\delta')}{\varepsilon^2 \cdot \log(n+1)}\right)$,
we have $s_{n,\varepsilon,\delta'} \geq c_1 \left(\frac{(n+1) \cdot \log(1/\delta')}{\varepsilon^2 \cdot \log(n+1)}\right)$
for some constant $c_1>0$ for all large enough~$n$.
Furthermore, we have set $\delta' = \frac{1}{\log\log\log n}$
and $\varepsilon \le \alpha \frac{1-\beta}{1+\beta}$.
Thus, we can continue the calculation as follows:
  \begin{eqnarray*}
        p_1 &\leq& \mathcal{O} \left( (n+1) ^{- \frac{\frac{c_1}{4}(n+1) \cdot \log(1/\delta')}{\varepsilon^2 \cdot \log(n+1) \cdot \left( 1+\sqrt{\log(n+1)} \right) }} \right)\\
        &\leq& \mathcal{O} \left( (n+1)^{- \frac{\frac{c_1}{4}(1+\beta)^2(n+1)\log\log\log\log n}{{\alpha}^2 (1-\beta)^2(\log(n+1) + (\log(n+1))^{(3/2)})}} \right)\\
        &=& \mathcal{O}\left( \frac{1}{\poly(n+1)} \right)
  \end{eqnarray*}
This final step holds because $(n + 1)\log\log\log\log n$ is $\omega(\log(n+1) + (\log(n+1))^{\frac{3}{2}})$,
and thus $p_1$ is at most $(n+1)^{-c_2}$ for some constant $c_2$ for all large enough~$n$.

    The probability that at least one of $s_1$ and $s_2$ exceeds half the sample size limit is at most $2p_1$,
    and this probability is an upper bound on $\delta_{poi}$. Therefore,
    $\delta_{poi} = \mathcal{O}\left( \frac{1}{\poly(n+1)} \right)$.
    
\end{proof}

\begin{lemma}
    \label{lemma:succeeds-rate}
    The probability that the two following events both occur is at least $1-\delta$, where $\delta =  \delta_{poi} + \delta'$. First, the sum of the two numbers $s_1,s_2$, drawn independently from a Poisson distribution with parameter $s_{n,\varepsilon,\delta'}/2$, is at most $s_{n,\varepsilon,\delta'} \cdot \left( 1 + \sqrt{\log(n + 1)} \right)$. Second, the estimate $\hat{L}_1(p,q)$ computed by the algorithm of Theorem~\ref{thm:estimated-L1} satisfies $|\hat{L}_1(p,q)-L_1(p,q)| \leq \varepsilon$.
\end{lemma}
\begin{proof}
    By Lemma~\ref{lemma:probability-sim-poisson}, the probability that the first
    event does not occur is at most $\delta_{poi}$.
    By Theorem~\ref{thm:estimated-L1}, the probability that the second event
    does not occur is at most $\delta'$. By the union bound, the probability
    that at least one of the two events does not occur is at most
    $\delta_{poi}+\delta'$. Thus, the probability that both events occur
    is at least $1-(\delta_{poi}+\delta')$.
\end{proof}

\subsection{Proof of Theorem~\ref{thm:competitive-ratio}}
\label{subsec:prove-thm}
We are now ready to prove Theorem~\ref{thm:competitive-ratio}.

\begin{proof}[Proof of Theorem~\ref{thm:competitive-ratio}]
We distinguish three cases regarding the relationship between $L_1(p,q)$ and $\tau$ as follows.

{\bf Case 1:} $L_1(p,q) \leq \tau - 2\varepsilon$ 

In this case, if the estimate $\hat{L}_1(p,q)$ satisfies the property $|\hat{L}_1(p,q) - L_1(p,q)|\le\varepsilon$, our
algorithm continues to run \textsc{Mimic} for the whole input as $\hat{L}_1(p,q) \leq L_1(p,q) + \varepsilon \leq (\tau - 2\varepsilon) + \varepsilon = \tau - \varepsilon$. By Lemma~\ref{lemma:mimic-estimated-matching-size}, \textsc{Mimic} produces a matching of size $\hat{n} - \frac{n}{2} L_1(p,q)$ in this case. Furthermore,
by Lemma~\ref{lemma:succeeds-rate}, the probability that the sample size $s_1+s_2$ does not exceed the sample size limit
$s_{n,\varepsilon,\delta'} \cdot \left( 1 + \sqrt{\log(n + 1)} \right)$
and $|\hat{L}_1(p,q) - L_1(p,q)|\le\varepsilon$ is at least $1-\delta$ with 
$\delta=\delta_{poi}+\delta'=\mathcal{O} \left( \frac{1}{\poly(n+1)} \right)+\frac{1}{\log \log \log n} =o(1)$, where $\delta_{poi}$ denotes the probability bound obtained in Lemma \ref{lemma:probability-sim-poisson} and $\delta'$ follows the definition in Lemmas~\ref{lemma:whp} and~\ref{lemma:probability-sim-poisson}.
Thus, the expected
size of the computed matching is at least $(1-\delta)\cdot \left(\hat{n} - \frac{n}{2} L_1(p,q)\right) = \hat{n} - \frac{n}{2} L_1(p,q) - o(\hat{n})\,$.
As the optimal matching size is at most $\hat{n} + \frac{n}{2} L_1(p,q)$ by Lemma~\ref{lemma:estimated-opt},
the competitive ratio is at least $\frac{\hat{n} - \frac{n}{2} L_1(p,q) - o(\hat{n})}{\hat{n} + \frac{n}{2} L_1(p,q)} = 1 - \frac{n L_1(p,q)}{\hat{n} + \frac{n}{2}L_1(p,q)}  - o(1) \geq 1 - \frac{2 L_1(p,q)}{2 \alpha + L_1(p,q)} - o(1)\,,$
where the final inequality holds as $\frac{\hat{n}}{n} \ge \alpha$.
Furthermore, by Lemma~\ref{lemma:set-threshold}, the competitive ratio is also at least $\beta-o(1)$.
In summary, for $L_1(p,q) \leq \tau - 2\varepsilon$, our algorithm obtains a competitive ratio
of $1 - \frac{2 L_1(p,q)}{2 \alpha + L_1(p,q)} - o(1)$, and the competitive ratio is also
at least $\beta - o(1)$.

{\bf Case 2:} $L_1(p,q) \geq \tau$ 
If the algorithm
switches to \textsc{Baseline} after observing $k \le s_{n,\varepsilon,\delta'} \cdot \left( 1 + \sqrt{\log(n + 1)} \right) $
arrivals, then,
by Corollary~\ref{cor:matching-size-switching-baseline}, the expected matching size is
at least $\beta \frac{n-k}{n} \cdot n^* - o(n^*)$.
The probability that the sample size $s_1+s_2$ does not exceed the sample size limit
and the estimate $\hat{L}_1(p,q)$ satisfies $|\hat{L}_1(p,q) - L_1(p,q)|\le\varepsilon$
is at least $1-\delta$.  Thus, the expected matching size is at least
$(1-\delta) \left (\beta \frac{n-k}{n} \cdot n^* - o(n^*)\right)
= \beta \frac{n-k}{n} \cdot n^* - o(n^*)$.

Therefore, the competitive ratio is 
\begin{eqnarray*}
\frac{\left( \beta \cdot \frac{n-k}{n} \cdot n^* \right) - o(n^*)}{n^*} &=& \beta \cdot \frac{n-k}{n} - o(1) \\
&\ge& \beta \cdot \left(1 - \frac{s_{n,\varepsilon,\delta'} \cdot \left( 1 + \sqrt{\log(n + 1)} \right) }{n} \right) - o(1) \\
&=& \beta - o(1)
\end{eqnarray*}

Thus, for $L_1(p,q) \geq \tau$, our algorithm achieves competitive ratio at least $\beta - o(1)$.

{\bf Case 3:} $L_1(p,q) \in (\tau-2\varepsilon, \tau)$

In this case, we show that our algorithm is $(\beta-o(1))$-competitive no matter whether
it runs \textsc{Mimic} or \textsc{Baseline}.
If the algorithm runs \textsc{Mimic} on the whole input, then by Lemma~\ref{lemma:set-threshold}, since $L_1(p,q) < \tau$, the competitive ratio is at least $\beta$.
If the algorithm runs \textsc{Baseline} on the whole input, it is $\beta$-competitive as \textsc{Baseline} is $\beta$-competitive.
If the algorithm switches to \textsc{Baseline} after the sampling phase, we can apply the same justification as in {\bf Case 2}
to show that it achieves competitive ratio at least $\beta - o(1)$.

\end{proof}

\section{Conclusions} \label{sec:conc}
We have studied online bipartite matching with predictions in the random arrival order model.
By extending the algorithm and analysis by Choo et al.~\cite{online-bipartite-matching-imperfect-advice},
we have shown that one can achieve $(1-o(1))$-consistency and $(\beta-o(1))$-robustness
provided that the predicted matching size $\hat{n}$ is at least $\alpha n$ for an arbitrary
constant $\alpha\in (0,1]$, no matter what the size $n^*$ of the optimal matching is.
In contrast, the analysis by Choo et al.~required $n^*=n$ and $\hat{n}\ge \beta n$.
Therefore, it does not establish $(1-o(1))$-consistency when $n^* < n$, even under perfect predictions.
In addition, we have shown that our algorithm has competitive ratio at least
$ 1 - \frac{2 L_1(p,q)}{2 \alpha + L_1(p,q)} - o(1)$, where $L_1(p,q)$ is the
prediction error, thus establishing smoothness.
Regarding future work, it would be interesting to investigate whether the results
can be generalized to the case where the predicted and optimal matching size are sublinear in~$n$.
It seems difficult to adapt the sampling-based approach to this case as an error
of $\varepsilon$ (for constant $\varepsilon$) in the estimate of the $L_1$-distance corresponds to a linear number of mispredicted vertices, which can be larger than the optimal matching size.

\bibliographystyle{plainurl}
\bibliography{references}

@InProceedings{online-bipartite-matching-imperfect-advice,
  title = 	 {Online bipartite matching with imperfect advice},
  author =       {Choo, Davin and Gouleakis, Themistoklis and Ling, Chun Kai and Bhattacharyya, Arnab},
  booktitle = 	 {Proceedings of the 41st International Conference on Machine Learning},
  pages = 	 {8762--8781},
  year = 	 {2024},
  editor = 	 {Salakhutdinov, Ruslan and Kolter, Zico and Heller, Katherine and Weller, Adrian and Oliver, Nuria and Scarlett, Jonathan and Berkenkamp, Felix},
  volume = 	 {235},
  series = 	 {Proceedings of Machine Learning Research},
  month = 	 {21--27 Jul},
  publisher =    {PMLR},
  url = 	 {https://proceedings.mlr.press/v235/choo24a.html}
}

@inproceedings{karp1990optimal,
  title={An optimal algorithm for on-line bipartite matching},
  author={Karp, Richard M and Vazirani, Umesh V and Vazirani, Vijay V},
  booktitle={Proceedings of the Twenty-Second Annual ACM Symposium on Theory of Computing},
  pages={352--358},
  year={1990}
}

@article{define-online-vertex-as-type,
  title={An experimental study of algorithms for online bipartite matching},
  author={Borodin, Allan and Karavasilis, Christodoulos and Pankratov, Denis},
  journal={Journal of Experimental Algorithmics (JEA)},
  volume={25},
  pages={1--37},
  year={2020},
  publisher={ACM New York, NY, USA}
}

@ARTICLE{estimated-l1-distance,
  author={Jiao, Jiantao and Han, Yanjun and Weissman, Tsachy},
  journal={IEEE Transactions on Information Theory}, 
  title={Minimax Estimation of the  $L_{1}$  Distance}, 
  year={2018},
  volume={64},
  number={10},
  pages={6672-6706},
  doi={10.1109/TIT.2018.2846245}}

@inproceedings{goel2008online,
  title={Online budgeted matching in random input models with applications to Adwords.},
  author={Goel, Gagan and Mehta, Aranyak},
  booktitle={SODA},
  volume={8},
  pages={982--991},
  year={2008}
}

@inproceedings{karande2011online,
  title={Online bipartite matching with unknown distributions},
  author={Karande, Chinmay and Mehta, Aranyak and Tripathi, Pushkar},
  booktitle={Proceedings of the Forty-Third Annual ACM Symposium on Theory of Computing},
  pages={587--596},
  year={2011}
}

@inproceedings{mahdian2011online,
  title={Online bipartite matching with random arrivals: {An} approach based on strongly factor-revealing {LPs}},
  author={Mahdian, Mohammad and Yan, Qiqi},
  booktitle={Proceedings of the forty-third annual ACM symposium on Theory of computing},
  pages={597--606},
  year={2011}
}

@article{manshadi2012online,
  title={Online stochastic matching: Online actions based on offline statistics},
  author={Manshadi, Vahideh H and Gharan, Shayan Oveis and Saberi, Amin},
  journal={Mathematics of Operations Research},
  volume={37},
  number={4},
  pages={559--573},
  year={2012},
  publisher={INFORMS}
}

@inproceedings{feldman2009online,
  title={Online stochastic matching: Beating 1-1/e},
  author={Feldman, Jon and Mehta, Aranyak and Mirrokni, Vahab and Muthukrishnan, Shan},
  booktitle={2009 50th Annual IEEE Symposium on Foundations of Computer Science},
  pages={117--126},
  year={2009},
  organization={IEEE}
}

@article{jaillet2014online,
  title={Online stochastic matching: New algorithms with better bounds},
  author={Jaillet, Patrick and Lu, Xin},
  journal={Mathematics of Operations Research},
  volume={39},
  number={3},
  pages={624--646},
  year={2014},
  publisher={INFORMS}
}

@inproceedings{brubach2016new,
  title={New algorithms, better bounds, and a novel model for online stochastic matching},
  author={Brubach, Brian and Sankararaman, Karthik Abinav and Srinivasan, Aravind and Xu, Pan},
  booktitle={24th Annual European Symposium on Algorithms (ESA 2016)},
  year={2016},
  organization={Schloss Dagstuhl-Leibniz-Zentrum fuer Informatik}
}

@article{brubach2020online,
  title={Online stochastic matching: New algorithms and bounds},
  author={Brubach, Brian and Sankararaman, Karthik Abinav and Srinivasan, Aravind and Xu, Pan},
  journal={Algorithmica},
  volume={82},
  number={10},
  pages={2737--2783},
  year={2020},
  publisher={Springer}
}

@article{aamand2022optimal,
  title={({Optimal}) Online Bipartite Matching with Degree Information},
  author={Aamand, Anders and Chen, Justin and Indyk, Piotr},
  journal={Advances in Neural Information Processing Systems},
  volume={35},
  pages={5724--5737},
  year={2022}
}

@inproceedings{li2023learning,
  title={Learning for edge-weighted online bipartite matching with robustness guarantees},
  author={Li, Pengfei and Yang, Jianyi and Ren, Shaolei},
  booktitle={International Conference on Machine Learning},
  pages={20276--20295},
  year={2023},
  organization={PMLR}
}

@article{purohit2018improving,
  title={Improving online algorithms via ML predictions},
  author={Purohit, Manish and Svitkina, Zoya and Kumar, Ravi},
  journal={Advances in Neural Information Processing Systems},
  volume={31},
  year={2018}
}

@inproceedings{gollapudi2019online,
  title={Online algorithms for rent-or-buy with expert advice},
  author={Gollapudi, Sreenivas and Panigrahi, Debmalya},
  booktitle={International Conference on Machine Learning},
  pages={2319--2327},
  year={2019},
  organization={PMLR}
}

@article{angelopoulos2024online,
  title={Online computation with untrusted advice},
  author={Angelopoulos, Spyros and D{\"u}rr, Christoph and Jin, Shendan and Kamali, Shahin and Renault, Marc},
  journal={Journal of Computer and System Sciences},
  volume={144},
  pages={103545},
  year={2024},
  publisher={Elsevier}
}

@inproceedings{anand2020customizing,
  title={Customizing ML predictions for online algorithms},
  author={Anand, Keerti and Ge, Rong and Panigrahi, Debmalya},
  booktitle={International Conference on Machine Learning},
  pages={303--313},
  year={2020},
  organization={PMLR}
}

@article{wei2020optimal,
  title={Optimal robustness-consistency trade-offs for learning-augmented online algorithms},
  author={Wei, Alexander and Zhang, Fred},
  journal={Advances in Neural Information Processing Systems},
  volume={33},
  pages={8042--8053},
  year={2020}
}

@article{lykouris2021competitive,
  title={Competitive caching with machine learned advice},
  author={Lykouris, Thodoris and Vassilvitskii, Sergei},
  journal={Journal of the ACM (JACM)},
  volume={68},
  number={4},
  pages={1--25},
  year={2021},
  publisher={ACM New York, NY}
}

@inproceedings{rohatgi2020near,
  title={Near-optimal bounds for online caching with machine learned advice},
  author={Rohatgi, Dhruv},
  booktitle={Proceedings of the Fourteenth Annual ACM-SIAM Symposium on Discrete Algorithms},
  pages={1834--1845},
  year={2020},
  organization={SIAM}
}

@article{wei2020better,
  title={Better and simpler learning-augmented online caching},
  author={Wei, Alexander},
  journal={arXiv preprint arXiv:2005.13716},
  year={2020}
}

@article{mitzenmacher2022algorithms,
  title={Algorithms with predictions},
  author={Mitzenmacher, Michael and Vassilvitskii, Sergei},
  journal={Communications of the ACM},
  volume={65},
  number={7},
  pages={33--35},
  year={2022},
  publisher={ACM New York, NY, USA}
}

@article{chung2003spectra,
  title={Spectra of random graphs with given expected degrees},
  author={Chung, Fan and Lu, Linyuan and Vu, Van},
  journal={Proceedings of the National Academy of Sciences},
  volume={100},
  number={11},
  pages={6313--6318},
  year={2003},
  publisher={National Acad Sciences}
}

@inproceedings{feng2021two,
  title={Two-stage stochastic matching with application to ride hailing},
  author={Feng, Yiding and Niazadeh, Rad and Saberi, Amin},
  booktitle={Proceedings of the 2021 ACM-SIAM Symposium on Discrete Algorithms (SODA)},
  pages={2862--2877},
  year={2021},
  organization={SIAM}
}

@article{jin2022online,
  title={Online bipartite matching with advice: Tight robustness-consistency tradeoffs for the two-stage model},
  author={Jin, Billy and Ma, Will},
  journal={Advances in Neural Information Processing Systems},
  volume={35},
  pages={14555--14567},
  year={2022}
}

@book{Motwani_Raghavan_1995, place={Cambridge}, title={Randomized Algorithms}, publisher={Cambridge University Press}, author={Motwani, Rajeev and Raghavan, Prabhakar}, year={1995}}

@inproceedings{NEURIPS2020_5a378f84,
 author = {Antoniadis, Antonios and Gouleakis, Themis and Kleer, Pieter and Kolev, Pavel},
 booktitle = {Advances in Neural Information Processing Systems},
 editor = {H. Larochelle and M. Ranzato and R. Hadsell and M.F. Balcan and H. Lin},
 pages = {7933--7944},
 publisher = {Curran Associates, Inc.},
 title = {Secretary and Online Matching Problems with Machine Learned Advice},
 url = {https://proceedings.neurips.cc/paper_files/paper/2020/file/5a378f8490c8d6af8647a753812f6e31-Paper.pdf},
 volume = {33},
 year = {2020}
}

@misc{predictions-github,
author={Alexander Lindermayr and Nicole Megow},
title={Algorithms with Predictions},
url={https://algorithms-with-predictions.github.io},
note={\href{https://algorithms-with-predictions.github.io}{algorithms-with-predictions.github.io}},
year=2025
}

@unpublished{Canonne2019,
author={Cl\'ement L. Canonne},
title={A short note on {Poisson} tail bounds},
note={Preprint},
year={2017}
}

\newpage
\appendix
\section{Notation}
\label{app-sec:notation}
The following is a list of the notation and some definitions used in the paper for easy reference.

\begin{itemize}
    \item $G = (U \cup V,E)$ - $G$ is the actual graph, where $U$ and $V$ are the sets of $n$ offline and $n$ online vertices, respectively.
    \item Type of online vertex $v \in V$ -- the set of nodes $\lbrace u \in U \mid \lbrace u,v \rbrace \in E \rbrace$.
    \item $c^*$ - function that maps each type $t$ to the number of type $t$ vertices in $V$.
    \item $c^*(t)$ - number of vertices in $V$ of type $t$.
    \item $T^*$ - set of types $t$ with $c^*(t) >0$.
    \item $r^*$ - size of $T^*$, i.e., $|T^*|$.
    \item $\hat{c}$ - function that maps each type $t$ to the number of online vertices of type $t$ predicted to arrive.
    \item $\hat{G} = (U \cup \hat{V},\hat{E})$ - Graph $\hat{G}$ is determined by the type prediction $\hat{c}$,  where $\hat{V}$ contains exactly $\hat{c}(t)$ vertices of type $t$.
    \item $\hat{c}(t)$ - number of vertices in $\hat{V}$ of type $t$ predicted to arrive.
    \item $\hat{T}$ - set of types $t$ with $\hat{c}(t) >0$.
    \item $\hat{r}$ - size of $\hat{T}$, i.e., $|\hat{T}|$.
    \item $L_1(c^*, \hat{c})$ - the $L_1$-distance between $c^*$ and $\hat{c}$.
    \item $p$ - actual input distribution of types $p = c*/n$.
    \item $q$ - predicted probability distribution of types $q = \hat{c}/n$.
    \item $L_1(p,q)$ - the $L_1$-distance between $p$ and $q$.
    \item $n$ - size of $U$, also size of $V$.
    \item $n^*$ - size of optimal matching in $G$. 
    \item $\hat{n}$ - size of the optimal matching in $\hat{G}$.
    \item $\beta$ -  the best-known competitive ratio for online bipartite matching in the random arrival order model without predictions. Currently $\beta=0.696$.
    \item $\tau$ - $\tau=\frac{2\hat{n}}{n}\cdot \frac{1 - \beta}{1+\beta}$
    \item $s_{n,\varepsilon,\delta'}$ - expected sample size based on $3$ parameters $n, \varepsilon,\delta'$. \\ $s_{n,\varepsilon,\delta'} = \Theta \left(\frac{(n+1) \log(1/\delta')}{\varepsilon^2 \log(n+1)}\right)$.
    \item $s_1 + s_2$ - actual sample size. $s_1,s_2$ drawn independently from a Poisson distribution with parameter $\frac12 s_{n,\varepsilon,\delta'}$.
    \item Sample size limit -- the value $s_{n,\varepsilon,\delta'}(1+\sqrt{\log(n+1)})$.
    \item $T^s_p$ - sample of $s_1 + s_2$ input vertices with replacement.
    \item $\hat{L}_1(p,q)$ - estimate of $L_1(p,q)$ based on the sample $T^s_p$ \\ such that $|\hat{L}_1(p,q) - L_1(p,q)| \leq \varepsilon$ with probability $1-\delta'$.
\end{itemize}

\end{document}